\documentclass[journal]{IEEEtran}
\ifCLASSOPTIONcompsoc
\usepackage[nocompress]{cite}
\else
\usepackage{cite}
\fi
\ifCLASSINFOpdf
\else
\fi
\usepackage{multirow,palatino, epsfig, latexsym, algorithmic, algorithm, epstopdf, amsmath, amssymb, color, amsthm, graphicx,enumitem, hyperref}

\newcommand{\ignore}[1]{}

\newtheorem{myTheorem}{Theorem}

\newtheorem{myDefinition}{Definition}

\newtheorem{myAssumption}{Assumption}

\usepackage[utf8]{inputenc} 
\usepackage[T1]{fontenc}    
\usepackage{hyperref}       
\usepackage{url}            
\usepackage{booktabs}       
\usepackage{amsfonts}       
\usepackage{nicefrac}       
\usepackage{microtype}      
\usepackage{amsmath}
\usepackage{algorithmic}
\usepackage{algorithm}
\usepackage{multirow,palatino, epsfig, latexsym, algorithmic, algorithm, epstopdf, amsmath, amssymb, color, amsthm,graphicx,enumitem, hyperref}

\begin{document}

\title{Autodecompose: A generative self-supervised model for semantic decomposition}

%

\author{Mohammad Reza Bonyadi
	
	\thanks{Mohammad Reza Bonyadi (rezabny@gmail.com).}
}

\IEEEtitleabstractindextext{%
\begin{abstract}

We introduce Autodecompose, a novel self-supervised generative model that decomposes data into two semantically independent properties: the desired property, which captures a specific aspect of the data (e.g. the voice in an audio signal), and the context property, which aggregates all other information (e.g. the content of the audio signal), without any labels given. Autodecompose uses two complementary augmentations, one that manipulates the context while preserving the desired property and the other that manipulates the desired property while preserving the context. The augmented variants of the data are encoded by two encoders and reconstructed by a decoder. We prove that one of the encoders embeds the desired property while the other embeds the context property. We apply Autodecompose to audio signals to encode sound source (human voice) and content. We pre-trained the model on YouTube and LibriSpeech datasets and fine-tuned in a self-supervised manner without exposing the labels. Our results showed that, using the sound source encoder of pre-trained Autodecompose, a linear classifier achieves F1 score of 97.6\% in recognizing the voice of 30 speakers using only 10 seconds of labeled samples, compared to 95.7\% for supervised models. Additionally, our experiments showed that Autodecompose is robust against overfitting even when a large model is pre-trained on a small dataset. A large Autodecompose model was pre-trained from scratch on 60 seconds of audio from 3 speakers achieved over 98.5\% F1 score in recognizing those three speakers in other unseen utterances. We finally show that the context encoder embeds information about the content of the speech and ignores the sound source information.

Our sample code for training the model, as well as examples for using the pre-trained models are available here: \url{https://github.com/rezabonyadi/autodecompose}
\end{abstract}
	
}

\maketitle
\IEEEdisplaynontitleabstractindextext
\IEEEpeerreviewmaketitle

\section{Introduction}
Formulating relevant features for semantically interesting classification tasks, such as classifying objects in images or speaker/speech recognition, has been a major undertaking for machine learning. While samples can be different from one another from many different aspects (two image might be different in terms of objects in them, or their background, or luminescence), it is usually the given task that dictates which aspects are more interesting. If labeled examples for the given task are available then a supervised approach could be taken into account. In complex tasks, however, the number of required labeled examples can be large to enable the model to learn a generalizable transformation from raw inputs to relevant features and to the labels, which is labor intensive and prone to over-fitting. An alternative is self-supervised learning which is used to create an overall "understanding" of the given data samples in the form a vector space (aka, embedding or representation space), usually through augmentations. Using the embedding space rather than the data in the raw form, it is expected that less number of labeled samples can be used to successfully distinguish between samples. The self-supervised approaches, however, do not guarantee establishing a semantically relevant embedding space that guarantee to be relevant to the classification task \cite{chen2020simple,tian2020makes}.

In this paper we propose a self-supervised generative approach which guarantees establishing an embedding space that encodes semantically meaningful and relevant features for a given classification task. For a given task, our approach requires two augmentations to be designed: One that randomly alters all "properties" in the raw data except the one that is desired to be used for classification, and another that maintains all properties in the data while randomly alters the one that is desired to be used for classification. These variations are fed into two separate encoders, and the output of the encoders are concatenated and used for reconstructing the original input. We prove that the proposed architecture guarantees encoding relevant information to the desired property (required for classification task) and all other properties separately. Our experiments showed that the encoded information in this self-supervised manner could be used for downstream tasks successfully, reducing the need for data labeling. We test our framework for characterizing audio signals, showing the ability of the method for the down-stream task of speaker recognition. As the proposed method is self-supervised, it can be used for characterizing sound sources it has never encountered before. 

\section{Related works}
We use the following notations throughout the paper:
We define the \textbf{observation space} $ \mathcal{D} $ from which data samples are observed, $ d \in \mathcal{D} $. 

\subsection{Self-supervised learning}
\label{sec:self-supervised}
Self-supervised learning aims to learn an embedding $ E:\mathcal{D}\to \mathbb{R}^d $ such that $ E(x) $ and $ E(x') $ are close (measured by a dissimilarity metric such as cosine similarity\cite{zimmermann2021contrastive,chen2020simple}) only if $ x\in \mathcal{D} $ is a data sample and $ x'\in \mathcal{D} $ is a variation of $ x $, usually generated by a set of augmentations\footnote{We define augmentation formally in next sections. An augmentation, $ A:\mathcal{D}\to\mathcal{D} $ is a stochastic function following a distribution, generating a new sample given a sample $ x \in \mathcal{D}$, in a way that some properties of the given sample are preserved.}. Optimizing the encoder $ E $ to ensure $ E(x) $ and $ E(x') $ are close may lead to a collapse of the embedding space, where all samples are mapped to one single vector. To avoid collapsed representations, two main families of approaches have been used: (i) contrasting samples \cite{chen2020simple,oord2018representation,tian2020contrastive,he2020momentum}; and (ii) diversification of representation space \cite{zbontar2021barlow,grill2020bootstrap,chen2021exploring,bojanowski2017unsupervised,richemond2020byol}.

Contrastive loss avoids collapse by introducing negative samples, i.e., the encoder is trained to bring $ x $ and its variation closer while push samples that are dissimilar to $ x $ further away. While this approach is effective, selection of negative samples is usually difficult. For example, there is no guarantee that a randomly chosen sample is dissimilar (in the sense of the desired down-stream task) to the given sample. Recent methods avoid collapse by introducing a secondary objective (regularization), usually added to the closeness metric, to ensure the embedded vectors are diversified in terms of how they use the dimensions in the embedding space. Similar results can achieve by introducing batch normalization (e.g., \cite{grill2020bootstrap,richemond2020byol}), centering and sharpening \cite{caron2021emerging}, or reconstructing the original sample\cite{liu2021self}. See \cite{liu2021self} for a survey on generative and contrastive self-supervised learning algorithms.

The augmentations applied for generating positive samples has shown to have an impact on the performance of the model when it applied to final downstream task \cite{chen2020simple,tian2020makes}. In other words, depending on the down-stream task, some approaches for generating positive (usually, augmentations) could be more effective than others. It was further found that self-supervised approaches may be sensitive to the features that are not important to the downstream task (e.g., in a dataset, most birds could be flying in the sky, which may lead to characterizing birds and sky together) \cite{zhao2021distilling}. 

There have been more recent efforts in the area of self-supervised learning. It has been shown that optimizing contrastrive loss is equivalent to finding the inverse of the generative process underlying data generation \cite{zimmermann2021contrastive}. This of course is tightly dependent on the augmentations used for generating the positive samples in the learning process. Further, it has been shown that the right augmentation process can lead to decomposition of style and content \cite{von2021self}. 

\subsection{Deep decomposition}
Autoencoders \cite{hinton2006reducing} represent the information in data points in a latent space and reconstructs the original data from the encoded vector. To improve the stability and robustness of the model, variational autoencoders (VAEs) were introduced \cite{kingma2013auto}. VAEs optimize for both the reconstruction of the input data (reconstruction loss) and the similarity between the latent space and a given prior distribution, measured by the KL divergence.

In a vanilla VAE, the KL divergence term encourages the latent space to be similar to the prior distribution, but the strength of this term is fixed. The $\beta$-VAE model \cite{higgins2016beta,higgins2021unsupervised} introduces a new hyperparameter, $\beta$, which controls the trade-off between reconstruction error and the KL divergence between the latent space and the prior distribution. Increasing the value of $\beta$ encourages the latent space to be more similar to the prior distribution, leading to more interpretable and disentangled latent representations. However, this may come at the cost of higher reconstruction error. The $\beta$ hyperparameter allows the model to explicitly control this trade-off and fine-tune the model's performance.

\ignore{
\subsection{Speaker models}
For identification ...

One limitation that comes with supervised models is that the trained model on a labeled dataset works well on that dataset while it might not show the same level of accuracy on other datasets, i.e., the sound features are domain sensitive [???] (see section ??? on why this is the case). 

While there has been a great progress in these areas through supervised methods [??, ??, ??], self-supervised methods have been applied only recently. 
}
\subsection{Voice and content decomposition}
Voice conversion (VC) is the process of generating an audio signal using a specific voice and content. In this paper, we focus on VC methods that use audio signals for the voice and content, and specifically on those that employ autoencoder architectures rather than generative adversarial networks or other techniques.

The goal of VC methods based on autoencoder architectures is to find two probabilistically independent representations for a given audio signal ($E_s$ and $E_c$), one representing the sound and the other representing the content (as shown in Fig. \ref{fig:vc_architecture}). These methods use various techniques, such as architectural choices (such as information bottlenecks and normalization) and providing labels for speakers or phonemes, to ensure independence between the information encoded by these encoders and to ensure that the sound and content are encoded by the appropriate encoder.

AutoVC \cite{qian2019autovc} is a VC model that uses two encoders to encode the sound ($E_s(.)$) and content ($E_c(.)$) of a given audio signal. The output of these encoders is then concatenated and fed to a decoder to reconstruct the original audio signal. This model optimizes two loss functions: a reconstruction loss, which ensures that the reconstructed signal is the same as the input, and a content maintenance loss, which ensures that the content is well preserved during the reconstruction process. To ensure that the sound and content are encoded independently in these encoders and to prevent any "leakage" of information between them, the input to the sound encoder ($E_s$) is taken from another audio sample of the same speaker. The content encoder ($E_c$) is designed with a bottleneck architecture (i.e., a small number of dimensions for the embedding) to reduce the risk of ignoring the sound encoder during the reconstruction process (as shown in Fig. \ref{fig:vc_architecture}(b)). Other techniques such as normalization are also used to improve performance.

\begin{figure}
	\centering
	\begin{tabular}{cc}
		\includegraphics[width=0.23\textwidth]{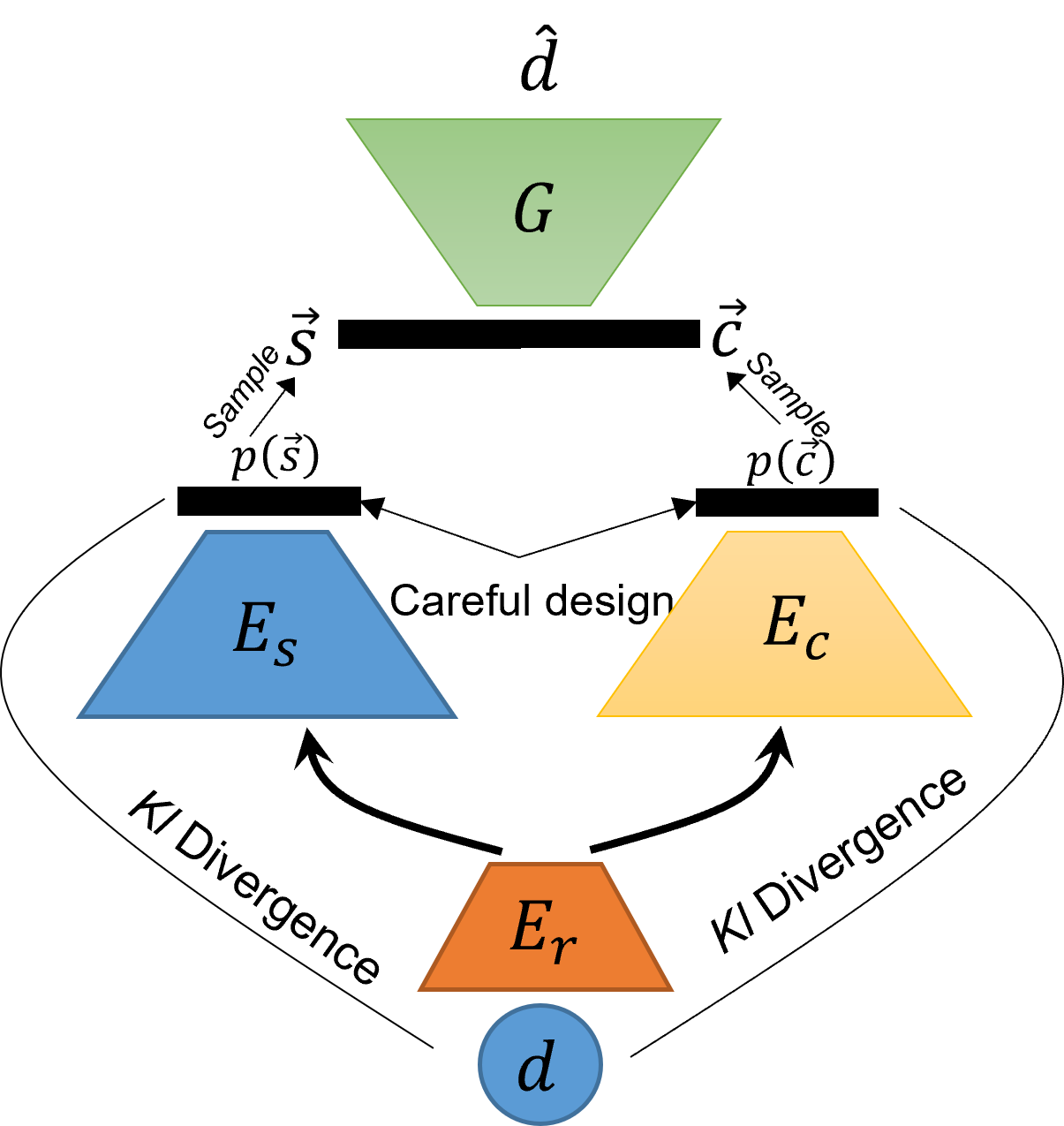}&\includegraphics[width=0.185\textwidth]{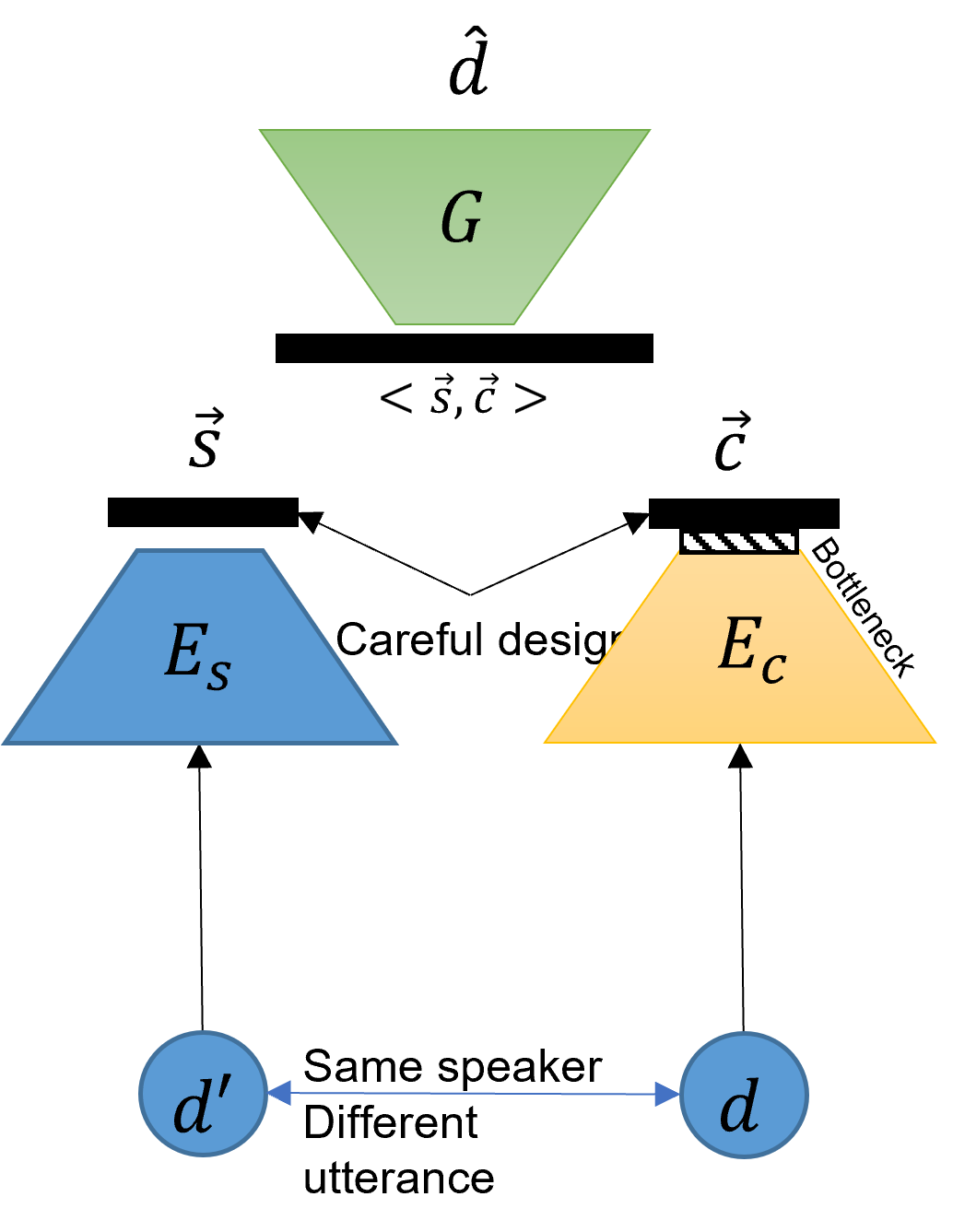}\\
		(a)&(b)
	\end{tabular}
	\centering	
	\caption{Architecture of different Voice conversion models.}
	\label{fig:vc_architecture}
\end{figure}

In the method introduced in \cite{lian2022robust}, three encoders were used: a shared encoder $E_r$, an encoder for embedding content ($E_c$), and an encoder for encoding speaker information ($E_s$). The content and speaker encoders were implemented as variational autoencoders, encoding probability distributions representing the content and speaker information in the input data. A decoder then took samples from these distributions and generated an estimation of the input audio signal. Two loss functions were optimized: one focused on reconstructing the input signal, and the other designed to ensure independence between the content and speaker probability distributions using KL divergence.

To ensure that the sound and content information were encoded by the appropriate encoders, different architectures were used for $E_s$ and $E_c$. Since the $E_s$ encoder was intended to encode time-independent voice information, its output was a vector. Normalization and bottlenecks were also used to ensure that sound information was carried by this encoder. The $E_c$ encoder, on the other hand, was intended to encode content information that is distributed in time, so its output was considered to be time-distributed. However, \cite{lian2022towards} later showed that the content latent space produced by this architecture was almost random and did not correlate with the phonemes of the content in the audio signals. To address this issue, a "phoneme-aware" architecture was proposed that conditioned the $E_c$ encoder on the phoneme labels of the audio signal, ensuring that the content embedding maintained content information at the phoneme level.

\section{Proposed approach}
We first provide relevant definitions and terminologies and then discuss our proposed approach.

\subsection{Our framework of definitions and assumptions}
Lets us define some terminologies we further use in this paper.

\begin{myDefinition}
	\label{def:basics}
	We define the \textbf{observation space} $ \mathcal{D} $ from which samples are observed, $ d \in \mathcal{D} $. We also define an \textbf{observer}, perceiving properties of samples by a set of functions, $ F^{(i)}: \mathcal{D} \to \mathcal{Z}^{(i)} $, $ i \in \{0, ..., k-1\} $, each function describing one \textbf{semantic property} of samples. Finally, we define the \textbf{observations dissimilarity metric}, $ L: \mathcal{D} \times \mathcal{D} \to \mathbb{R} $, as a dissimilarity function, $ L(., .) $ is bounded from below, and $ L(d, \hat{d}) $ is at its minimum if and only if $ F^{(i)}(d)=F^{(i)}(\hat{d}) $ for all $ i $, $ d $ and $\hat{d}$ are arbitrary samples from $ \mathcal{D} $. 
\end{myDefinition}

The semantic properties form the way an observer (usually a human) perceives samples in $ \mathcal{D} $, i.e., they are semantically meaningful for the observer. The observation dissimilarity metric measures the extent to which the observer perceives $ d $ and $\hat{d}$ as similar. This means if we alter a sample $ d $ in a way that the semantic properties remain unchanged then the altered sample and the original sample are not distinguishable from the observer perspective, i.e., semantically the same. In other words, if observation dissimilarity metric is at its minimum then $ d $ and $\hat{d}$ are not distinguishable from observer perspective. 

We consider the following assumptions for the observable properties.

\begin{myAssumption}	
	\label{ass:basic_assumptions}
	
	\textbf{(a)} For any given sample $ d $, the value of all semantic properties can be determined by the observer.
	
	\textbf{(b)} The functions $ F^{(i)} $ for semantic properties are not explicitly known. 
	
	\textbf{(c)} We assume $ \mathcal{Z}^{(i)} $ is a vector space $ \mathbb{R}^{n_i} $, $ n_i \in \mathbb{N} $. 
	
	\textbf{(d)} We assume $ k=2 $, where $ F(d) = F^{(0)}(d) $, $ F: \mathcal{D} \to \mathbb{R}^n $, is called \textbf{the desirable semantic property} (DSP) and $ F'(d) $, $ F': \mathcal{D} \to \mathbb{R}^{n'} $, is an arbitrary aggregation of all other semantic properties, called the \textbf{context semantic property} (CSP). 
	
	\textbf{(e)} DSP is mutually independent from all other semantic properties (and, hence, from the CSP), i.e., given $ F(d) $, $ F'(d) $ cannot be calculated and vice versa.  
\end{myAssumption}

Assumption \ref{ass:basic_assumptions}(a) and (b) indicate that the observer has a complete understanding of the semantic properties, although the mathematical function of these properties is not accessible (e.g., human can determine objects in an image, but the mathematical formula for performing this task is not accessible). Regarding assumption \ref{ass:basic_assumptions}(d), we consider the first property ($ F^{(0)}(.) $) describes one aspect of the samples (we call this aspect the \textit{desirable semantic property, DSP} throughout the paper) and the other property describes an arbitrary aggregation of all other aspects (the context for that property, CSP). Assumption \ref{ass:basic_assumptions}(e) ensures that property values do not replicate one other, i.e., given the value for one property, the value for the other property cannot be calculated for all $ d $. In other words, properties $ F(d) $ and $ F'(d) $ encode mutually independent information \footnote{An alternative measure for this independence can be the K-L divergence.}.

An example for the observation space is the space of all possible speech signals. Example of DSP is the sound source in the audio signals (e.g., person voice), where CSP would be the content of the audio signal (including the phoneme, emotion, rhythm, among other properties). Given the sound source in an audio signal, the content cannot be recovered, and vice versa (Assumption \ref{ass:basic_assumptions}(e)). An example of observations dissimilarity measure is the euclidean distance between samples of two speech signals, i.e., if the euclidean distance is zero then an observer would not be able to distinguish between the two audio signals. Other distances could be also used, given the specific use case. For example, in an object recognition task from images, euclidean distance is minimized if all semantic properties are equal for two given samples. This, however, can be too restrictive as it compares images at the pixel level rather than semantics. Hence, in complex tasks, the dissimilarity metric could be designed by the observer.

\ignore{
\begin{myDefinition}
We say the function A completely manipulates property F of d if and only if F(d) cannot be estimated to any arbitrary error given all a from A(d)
We say two functions H and H' are \textit{independent} iff there is no function R such that R(H(x)) = H'(x)
a and d are F-independent 
there is no R such that for all d, R(A(d))=F(d)
\end{myDefinition}
}
We finally define augmentations and complement of an augmentation as follows:

\begin{myDefinition}
	\label{def:augmentations}
	We define an \textbf{augmentation} as a stochastic function $ A: \mathcal{D} \to \mathcal{D} $. For any given sample $ d $, $ a \sim A(d) $ where $ F(d)=F(a) $ and $ F'(d) $ and $ F'(a) $ are independent, i.e., $ F'(d) $ cannot be calculated given $ a $. 
	
	We also define $ A': \mathcal{D} \to \mathcal{D} $, the \textbf{complement} of $ A(.) $ as a stochastic function. For any given sample $ d $, $ a' \sim A'(d) $ where $ F'(d) = F'(a') $ and $ F(d) $ is independent of $ F(a') $, i.e., $ F(d) $ cannot be calculated given $ a' $.  
\end{myDefinition}

In other words, the augmentation $ A(.) $ randomly (with some distribution) manipulates all semantic properties for a given sample except for the DSP ($ F(d)=F(A(d)) $ For all $ d $). In contrast, the augmentation $ A'(.) $ randomly (with some distribution) manipulates the DSP for a given sample and leaves the context properties untouched. Note that, we assume $ A $ manipulates the sample $ d $ in a way that $ F'(d) $ cannot be calculated given $ a $. Also, we assume $ A' $ manipulates the sample $ d $ in a way that $ F(d) $ cannot be calculated given $ a $.

The definition of augmentation here extends the definition provided in \cite{von2021self} by assuming that the latent space controlling the augmentation function is decomposable to two sub-spaces, $ \mathbb{R}^n \times \mathbb{R}^{n'} $, one of which represents a specific property of the sample (DSP) and the other represent the "context" for that property (CSP). This allows us to manipulate a particular property or its context independently. It also allows us to investigate the decomposition capabilities of different algorithms.

Our definition here can also assist with explaining the observation that some positive samples generated by some specific types of augmentation are more suited for some final downstream tasks \cite{chen2020simple,tian2020makes}. Indeed, one can hypothesis that an augmentation which preserves data properties that are related to the downstream task could improve the performance for that task.

\subsection{Autodecompose: An augmentation-driven learning}
In this paper we propose Autodecompose as an alternative approach for estimating the desirable property $ F(.) $. Autodecompose is composed of two encoders $ E: \mathcal{D} \to \mathbb{R}^n $ and $ E': \mathcal{D} \to \mathbb{R}^{n'} $ and a decoder $ G: \mathbb{R}^{n \times n'} \to \mathcal{D} $ such that $ L\left(d, \hat{d}\right) $ is minimized for all $ d \in \mathcal{D} $, where $ L(., .) $ is a dissimilarity measure (see definition \ref{def:basics}), $ d $ is a sample from $\mathcal{D}$, $ \hat{d} = G\left(<E(a), E'(a')>\right) $, $ a \sim A(d) $, and $ a' \sim A'(d) $. The overall architecture for Autodecompose has been shown in Fig. \ref{fig:autodecompose_basic}. 

\begin{figure}
	\label{fig:autodecompose_basic}
	\centering
	\includegraphics[width=0.35\textwidth]{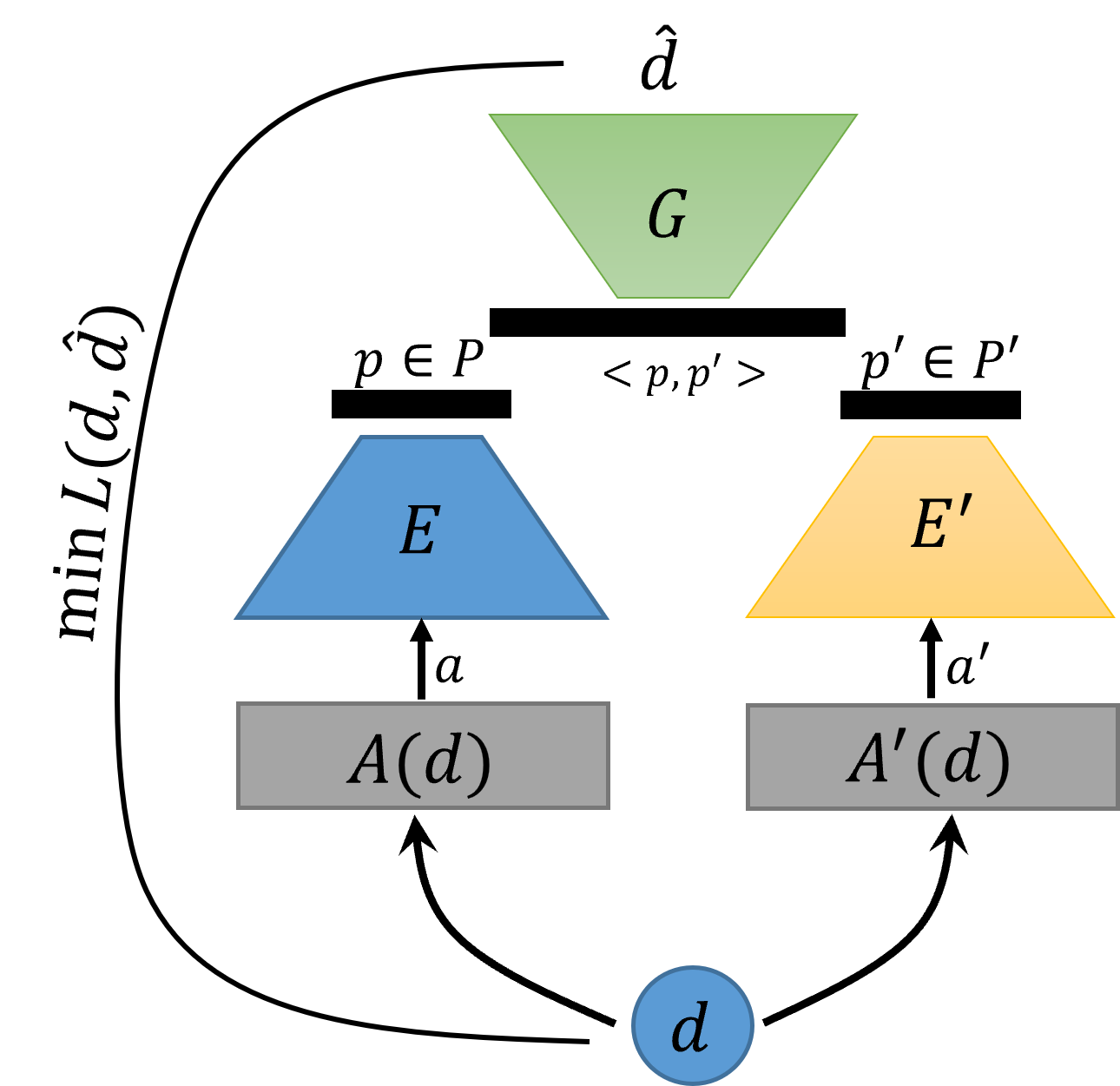}
	\centering	
	\caption{Proposed autodecompose architecture.}
\end{figure}

With this setting, we prove that if the encoders and the decoder guarantee minimizing $ L(d, \hat{d}) $ for any $ d \in \mathcal{D} $ then $ E(d) $ is a representation of $ F(d) $ and $ E'(d) $ is a representation of $ F'(d) $. 

\begin{myTheorem}
	\label{thr:main_theorem}	
	Assume $ L\left(d, \hat{d}\right) $ is at its minimum for all $ d \in \mathcal{D} $, where $ \hat{d} = G\left(<e, e'>\right) $, $ e=E(a) $, $ e'=E'(a') $, $ a \sim A(d) $ and $ a' \sim A'(d) $. We prove that there exists a function $H: \mathcal{P} \to \mathcal{Z}$ for all $ d \in \mathcal{D} $ such that $ H(E(d))=F(d)$.
\end{myTheorem}
\begin{proof}	
	\textbf{Proof by contradiction}.
	
	Based on Definition \ref{def:basics}, if $ L(\hat{d}, d) $ is at its minimum then $ F(d)=F(\hat{d}) $ and $ F'(d)=F'(\hat{d}) $. 
	
	\ignore{
		Based on Definition \ref{def:augmentations}, $ A(.) $ randomly manipulates the value of $ F'(d) $ (information about CSP is manipulated by $ A $, i.e., $ F'(d) \ne F'(A(d))$). Because $ E(.) $ is deterministic, there is no function $ H:\mathcal{P} \to \mathcal{Z} $ and encoder $ E(.) $ such that $ H(E(a)) = F'(d) $. In other words, there is no encoder $ E(a) $ that can extract $ F'(d) $ for all $ d $.}
	
	Based on Definition \ref{def:augmentations}, $ A'(.) $ randomly manipulates the value of $ F(d) $. As $ E(.) $ is a deterministic transformation, $ F(d) $ cannot be calculated given $ E'(a') $. 
	
	We assume (contradiction assumption) there is not function $ H: \mathcal{P} \to \mathcal{Z} $ such that $ H(E(a)) = F(a)$ for all $ d $. In this case, $ E'(a') $ and $ E(a) $ are both independent of $ F(d) $. Hence, $ G(<E(a), E'(a')>) $ would be independent of $ F(d) $. Therefore, $ F(d)=F(\hat{d}) $ cannot be true, which is in contradict with the assumption that $ L(\hat{d}, d) $ is at its minimum. Thus, the contradicting assumption was not correct and there exists a function $H: P \to \mathcal{Z}$ for all $ d \in \mathcal{D} $ such that $ H(E(d))=F(d)$. 
\end{proof}

\ignore{
	
	There is no function R such that, for all d, R(a) = F'(d)
	There is no function R such that, for all d, R(a') = F(d)
	Hence, 
	There is no function R such that, for all d, R(e) = F'(d)
	There is no function R such that, for all d, R(e') = F(d)
	
	We assume (contradictory) there is no function R such that R(e) = F(d)
	
	This means none of e and e' can be used 
	
\begin{myTheorem}

	Assume $ L\left(x, \hat{x}\right) $ is at its minimum for all $ x \in \mathcal{D} $, where $ \hat{x} = G\left(<e, e'>\right) $, $ e=E(a) $, $ e'=E'(a') $, $ a \sim A(x) $ and $ a' \sim A'(x) $. 
	
	We prove that 
	
	$ \forall x_1, x_2 \in \mathcal{D}, ~ E(x_1)=E(x_2) \implies F(x_1)=F(x_2)$
\end{myTheorem}
\begin{proof}	
	\textbf{Proof by contradiction}.
	
	Based on Definition \ref{def:basics}, if $ L(\hat{d}, d) $ is at its minimum then $ F(d)=F(\hat{d}) $ and $ F'(d)=F'(\hat{d}) $. 
	
	Based on Definition \ref{def:augmentations}, $ A'(.) $ randomly manipulates the value of $ F(.) $. Hence, the property $ F(.) $ cannot be encoded by $ E'(.) $.
	
	We assume (contradictory assumption) there exists two samples $ x_1, x_2 \in \mathcal{D} $ for which $ E(x_1)=E(x_2) $ while $ F(x_1)\neq F(x_2) $. With this assumption, and given that $ F(x_i) $ cannot be retrieved given $ a_i' $, $ G\left(<e_i, e'_i>\right) $ would not be able to fully reconstruct (minimize $ L(., .) $) for $ x_1 $ or $ x_2 $ (all inputs to $ G(.) $ are independent of $ F(x_1) $ and $ F(x_2) $). This is in contradict with the assumption that $ L(\hat{x}, x) $ is at its minimum for all $ x $.
	
\end{proof}
}

Intuitively, $ \hat{d}=G(<e, e'>) $ can be similar to $ d $ (reconstructed to minimize $ L(., .) $) only if $ <e, e'> $ carries complete information about the properties $ F(d) $ and $ F'(d) $. The Theorem proves that all information related $ F(d) $ is carried only by $ e $ and not by $ e' $, and all information related to $ F'(d) $ is carried by $ e' $ and not by $ e $. A simple reason is that $ e $ is independent of $ F'(.) $ because $ A(.) $ manipulates $ F'(.) $. According to Theorem \ref{thr:main_theorem}, an optimal encoder $ E(.) $ would guarantee providing a representation of $ F(.) $ for all samples, as, otherwise, the loss function $ L(., .) $ could not be minimized. A similar proof can be provided to show that the $ F'(.) $ is represented by $ E'(.) $. 

There is no reason to assume that $ E(.) $ is a linear function of $ F(.) $. One should also note that $ E $ cannot provide any representation for $ F'(.) $ as this property is not accessible to $ E $ (randomly manipulated by $ A $). The same is true for $ E' $ and $ F(.) $. In other words, $ E $ and $ E' $ establish separate representations for DSP and CSP (we will test this experimentally in the Experiments section). While there is no reason to assume these representations are linear relationships, our experiments show that these relationships are close to linear (i.e., $ E $ and $ E' $ are linear representations of $ F $ and $ F' $).

\textbf{Role of the generator}: The generator $ G(.) $ has an important role here. Augmentations $ A $ and $ A' $ cannot be necessarily used as the positive and negative instances in contrastive loss or self-supervised frameworks like simCLR. The reason is that $ A(d) $ does not represent $ F(d) $, but only provides samples for which $ F(a)=F(d) $, given any $ d $. Hence, using contrastive loss with negative samples generated by $ A' $ and positive samples generated by $ A $ would lead to learning these augmentation functions rather than the DSP and CSP. Consider, for example, $ A' $ masks some parts of the input while $ A $ doesn't (and performs some other type of manipulation). A constrastive loss then may learn that the existance of masked areas in the input translates to negative samples, which is not related to DSP or CSP. The generator in Autodecompose, however, ensures the captured features in the encoders represent the actual DSP and CSP as, otherwise, reconstruction would not be successful. One hypothesis here is that the generator $ G $ makes a connection between the "mechanics" of augmentations and how they relate to reality because, ultimately, $ G $ needs to translate the encoded augmented signals to the original format for final dissimilarity comparison. This leads to the critical role of the generator, which is responsible for ensuring $ F $ and $ F' $ are encoded by $ E $ and $ E' $, as proven by the Theorem \ref{thr:main_theorem}. This is true despite the fact that $ A $ and $ A' $ do not guarantee generating all possible samples with similar properties for a given sample. 

\begin{figure}
	\label{fig:overall}
	\centering
	\includegraphics[width=0.45\textwidth]{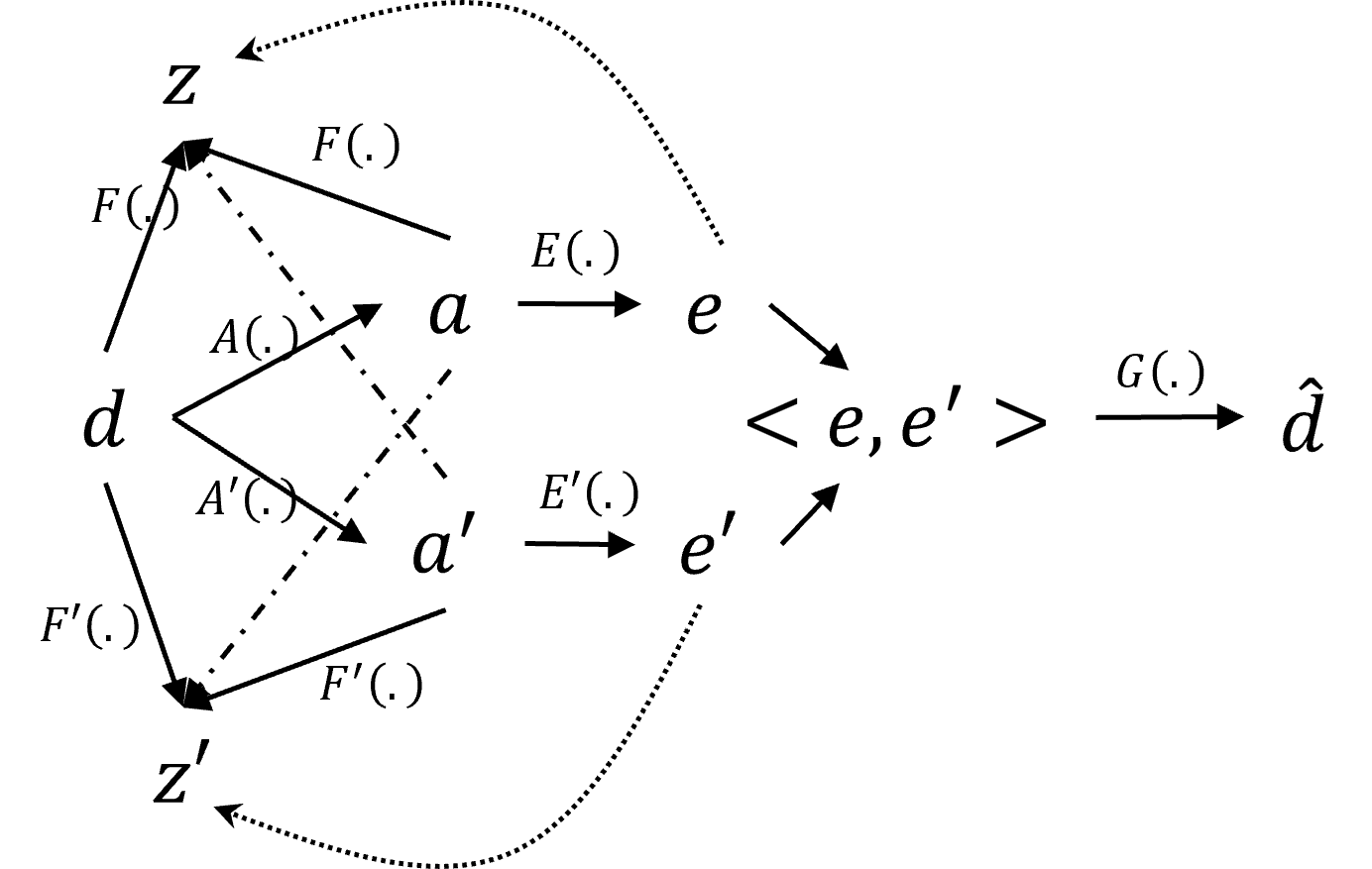}
	\centering	
	\caption{The dot-dashed line indicates that there is no such function that maps $ a $ to $ z' $ and $ a' $ to $ z $. The dotted line indicates that $ E(.) $ represents $ F(.) $ and $ E'(.) $ represents $ F'(.) $ (see Theorem).}
\end{figure}

Figure \ref{fig:overall} shows the relationships between all variables and functions discussed above.

\textbf{Implementation}: For implementation of autodecompose we consider $ E(.) $, $E'(.)$, and $ G(.) $ as parameterized functions and we optimize their parameters with the following loss function:

\begin{eqnarray}
	\label{eq:autodecom_base1}
	\min_{\omega, \theta, \theta'} L\left(d, G_{\omega}\left(<E_{\theta}(a), E'_{\theta'}(a')>\right)\right)
\end{eqnarray}

where $ L(., .) $ is a dissimilarity measure (i.e., reconstruction loss, adversial loss), $ d $ is sampled from $\mathcal{D}$, $ E: \mathcal{D} \to P $ and $ E': \mathcal{D} \to P' $ are encoders, $ G: P \times P' \to \mathcal{D} $ is a generator, and $ P $ and $ P' $ are two vector spaces, $ a \sim A(d) $ and $ a' \sim A'(d) $, and we use an optimization method to find optimal values for $ \omega $, $ \theta $, and $ \theta' $ given some samples of $d\in \mathcal{D}$. 

\subsection{Differences with self-supervised learning and VC}
The decomposition approaches, such as VAE and $\beta$-VAE, guarantee establishing an embedding space, each dimension representing an independent feature of the data points. There is, however, no guarantee that the mutually independent features are semantically meaningful, which is the main difference between those approaches and autodecompose. The same limitation applies to VC algorithms. Different VC algorithms ensure semantic relevance of the encoded features by carefully selecting the architecture of the encoders or providing labeled instances. 

Self-supervised approaches use augmentations to train basic models, similar to our autodecompose. The main difference is, however, that the augmentations are not picked according to semantics required for learning features related to DSP. Unlike typical self-supervised methods (see Section \ref{sec:self-supervised}), autodecompose does not suffer from collapse. The reason is that the generator $ G $ mirrors the original samples in a way that the observer could not distinguish between the two (original and the mirrored), ensuring to encode semantically meaningful features of the observations by the encoders, controlled by the provided augmentations. Collapse usually takes place when mapping to a single point in the embedding space can satisfy the loss function, which is not possible in the autodecompose architecture. Typical self-supervised approaches usually use negative samples (e.g., simCLR) or diversification techniques to ensure diversity of the instances is maintained and avoid collapse. 

\section{Autodecompose case study: Audio signals}
In this paper, we adopt an autodecompose model that decomposes a given audio signal to its sound source and content, corresponding with $ E_s(.) $ and and $ E_c(.) $, respectively. We optimize the following loss function:
 
\begin{eqnarray}
	\min_{\omega, \theta, \theta'} L\left(d, G\left(E_s(\vec{s}; \theta), E_c(\vec{c}; \theta'); \omega\right)\right)
\end{eqnarray}
where $ A_s:\mathcal{D} \to \mathcal{D} $ and $ A_c: \mathcal{D} \to \mathcal{D} $ are two complementary augmentations (see definition \ref{def:augmentations}), $ \vec{s} \sim A_s(d) $ and $ \vec{c} \sim A_c(d) $. For our purposes in this article (encoding content and sound separately given an audio signal), we design $ A_s(d) $ in a way that it randomly manipulates the content of $ d $ while maintains the sound source. In contrast, we design $ A_c(d) $ in a way that it randomly manipulates the sound source while maintains everything else (i.e., content). With this design, the optimal parameters for $ E_s(.; \theta) $ and $ E_c(.; \theta') $ would encode the sound and the content information separately, as proven in Theorem. 

\begin{figure}
	\centering
	\includegraphics[width=0.35\textwidth]{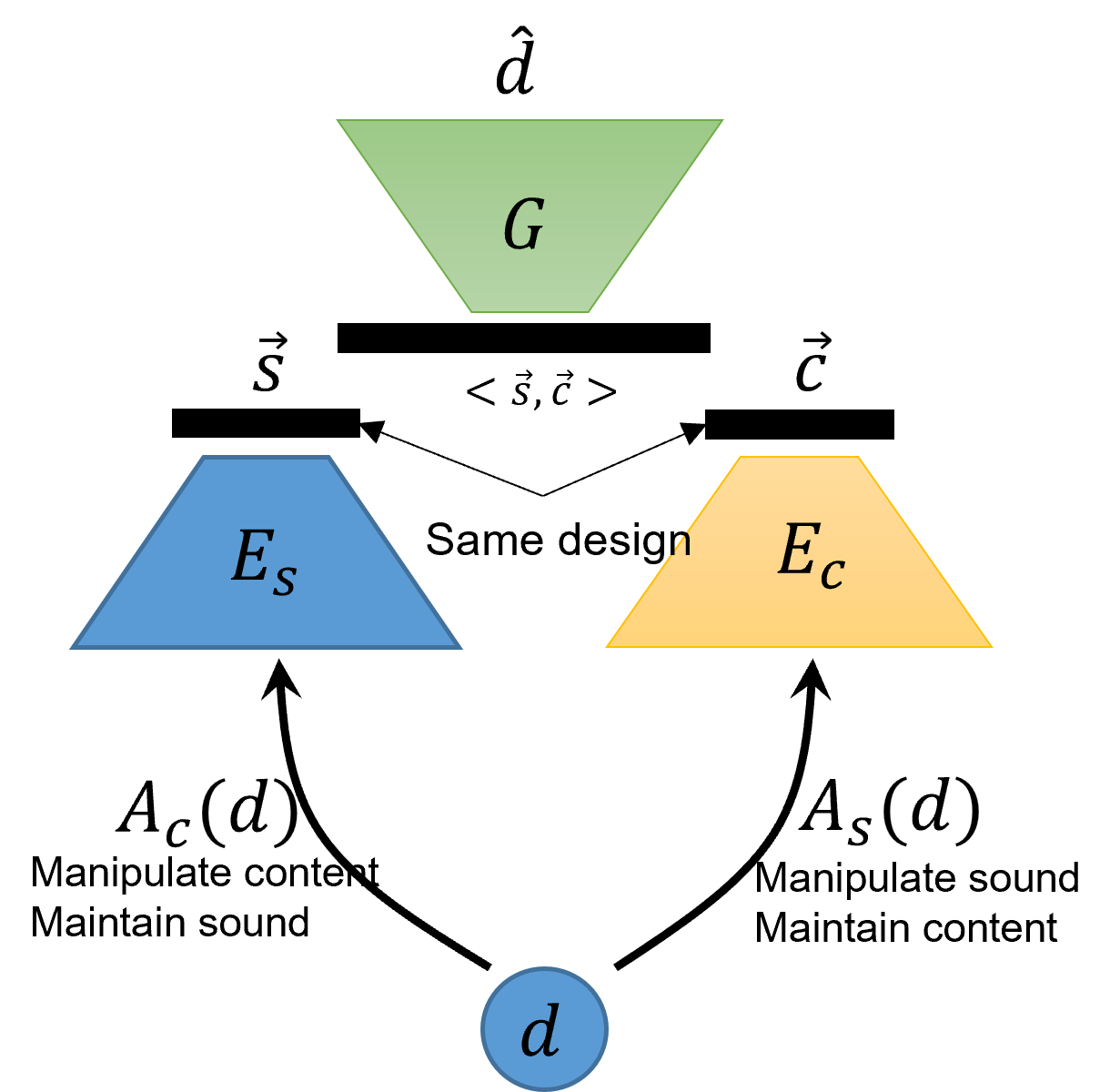}
	\centering	
	\caption{Autodecompose architecture for speaker recognition.}
\end{figure}

We use the mel-spectrum of the audio signals as input.

\subsection{Design of complementary augmentations}
One potential way for designing $ A_s $ is to use data labels: For a given audio signal $ d $ with a specific speaker, $ A_s(d) $ would be another utterance from the same speaker. This, however, would need pre-labeling the data or rely on assumptions (e.g., the next 1 second of the current utterance is likely to be uttered by the same speaker). We propose two label-independent augmentations that maintain the sound source while change the content of the audio signal: \textbf{Scrambling} mel spectrum along the time dimension, \textbf{masking} some parts of the mel spectrum along the time dimension (see Fig. \ref{fig:a_s}). 

\textbf{Time-domain scrambling}: The scrambling is done by selecting a random pivot point in the spectrum along the time axis and inverting the two segments. This is performed multiple times (randomly picked every time with a uniform random number between 5 to 20) with different random pivot points. This augmentation maintains the sound source while changes the content. The larger the number of random pivots, the more the content would be different from the original. 

\textbf{Time-domain masking}: Another augmentation we found effective was to select a segment of the mel-spectrum along the time axis and replace that segment with zeros. Our experiments showed that this augmentation can also add value if it is done for only few segments (removing 2 segments in our experiments) with short lengths each (two frames of the mel spectrum). 

Using these two augmentations, the content of $ d $ would be manipulated but the sound source would remain untouched, i.e., these two augmentations generate a new content (although meaningless) uttered by the same speaker. In terms of calculations, these two augmentations are quick to calculate and apply. 

\begin{figure}
	\label{fig:a_s}
	\centering
	\includegraphics[width=0.45\textwidth]{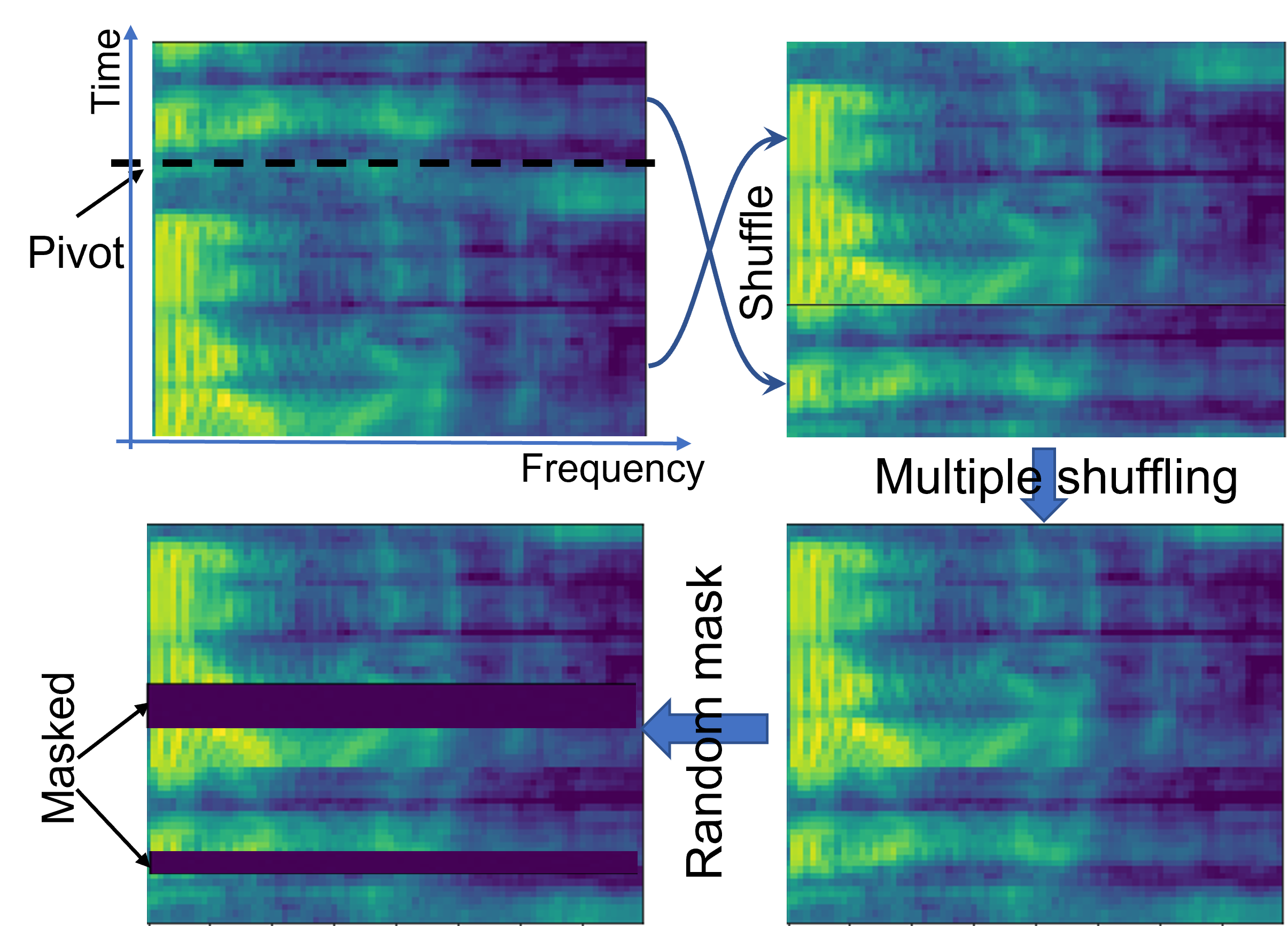}
	\centering	
	\caption{Augmentation $ A_s $}
\end{figure}

To design $ A_c(d) $, one can consider using the phoneme data and find the same word uttered by another speaker. This, however, requires accessing the text of the utterances. We propose two manipulations that are independent of any labeling: Frequency-domain stretch/shrink and masking frequency bands (see Fig. \ref{fig:a_c})

\textbf{Frequency-domain shrinking/stretching}: The shrinking or stretching along the frequency dimension leads to changing the audio pitch, which would change the perceived sound source (voice). To do so, we first stretch/shrink the mel-spectrum along the frequency dimension by some percentage and then resample the new spectrum to the original number of mel channels using an interpolation algorithm (used cubic Spline in our implementations). This would impact change the pitch of the voice. We stretch/shrink the mel spectrum randomly by maximum of $ 15\% $ and minimum of $ 2\% $, sampled from a uniform distribution. 

\textbf{Frequency-domain masking}: We randomly select a few segments (up to 15 segments), with maximum length of 5 mel bins along frequency dimension and replace them by zeros. The low frequency mel bins (10 bins corresponding to the lowest frequency information) were left out of this manipulation as low frequency information carry most of the content. 

\begin{figure}
	\label{fig:a_c}
	\centering
	\includegraphics[width=0.45\textwidth]{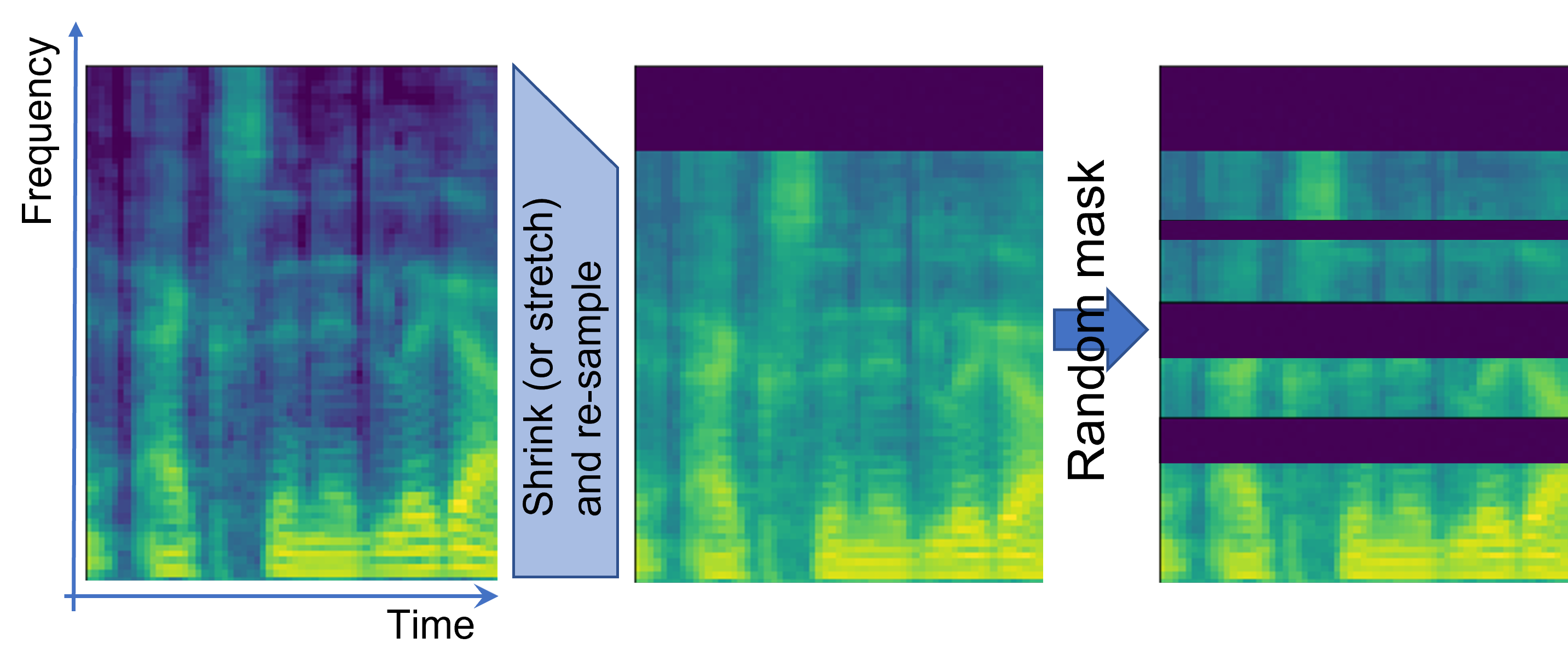}
	\centering	
	\caption{Augmentation $ A_c $}
\end{figure}

\subsection{Overfitting}

The stochastic nature of the complementary augmentations used for autodecompose is expected to combat overfitting, making this approach effective even if the sample size is small comparing to the complexity of the model. We test this hypothesis further in the paper where we use a large model with large embedding capacity on a small amount of data.

\section{Experiments}
We provide results of multiple experiments where we used autodecompose for a speaker recognition task and for content encoding task.

\subsection{Overall experiments setup}

\textbf{Datasets}: We used two datasets for our training and testing purposes: LibriSpeech train-clean-100 and test-other datasets, as well as a set of hand-picked YouTube videos all containing some meetings or panel discussions with variety of qualities and languages. YouTube dataset was used for pre-training only, while the Librispeech datasets were used for pre-training, fine tuning, and testing in different experiments. The total length of the audios extracted from YouTube videos was close to 24 hours, included publicly available meetings sessions and panel discussions. 

\textbf{Data preparations}: We first re-sampled the raw audio signals in $ 16k $ hz sampling rate and then filtered them with a band pass filter in the interval $ [90, 7600] $ hz. The mel spectrum was calculated per audio signal using number of FFT of 256, hop-length of 256, FFT length of 256, and number of mel channels equal to 80. This generates a mel-spectrum, each \textit{time bin} is 0.016 seconds, represented by 256 samples, in 80 \textit{frequency bins} (mel), and 0.0 seconds of overlap with the next time bin. We then cropped this mel spectrum in the time dimension over every 64 time bins of the spectrum, generating spectrums with dimension of 64 by 80 (corresponding to 1.024 second of audio signal). 

\textbf{Architectures}: We experimented with multiple architectures for our model. We built those architecture by combining different encoder and decoder choices. In all tests, we used the same architecture for the DSP and CSP encoders ($ E_s $ and $ E_c $). Here are the architecture choices we used:

\begin{itemize}
	\item Large encoder: 3 conv. (512 filters) + BNR + 3 LSTM (256 units) + embd128
	\item LSTM encoder: 3 LSTM (256 units) + embd128
	\item Conv encoder: 3 conv. (512 filters) + BNR + embd128
	\item Dense encoder: 1 dense (512 neurons) + BNR + embd128
	\item Large decoder: 2 conv. (512 filters) + BNR + 2 LSTM (512 units) + out
	\item LSTM decoder: 2 LSTM (512 units) + out
	\item Conv decoder: 2 conv. (512 filters) + BNR + out 
	\item Dense decoder: 1 dense (1024 neurons) + out
\end{itemize}

In these architectures, embd128 refers to dense layer (128 dimensions) with linear activation, BNR refers to batch norm + ReLU, out refers to dense (80 neurons), "conv." refers to one-dimensional convolutional layers, "LSTM" refers to stacked long short-term memory layers, and "dense" refers to dense layers.

\subsection{Comparison across architectures}
For each architecture, we first trained the autodecompose model on utterances from 10 speakers taken randomly from LibriSpeech dataset \cite{panayotov2015librispeech} (this group of people remained the same for all architectures), train-clean 100, without any labels. We then used 10 seconds of each person's voice, selected randomly of the utterances, to train a classifier (logistic regression) on the embedded signals and left the rest of the data for testing the classifier. We pre-training the models for 50 steps and used $ E_c $ encoder. Results are reported in Figure \ref{fig:architectures}.

\begin{figure}
	\centering
	\includegraphics[width=0.45\textwidth]{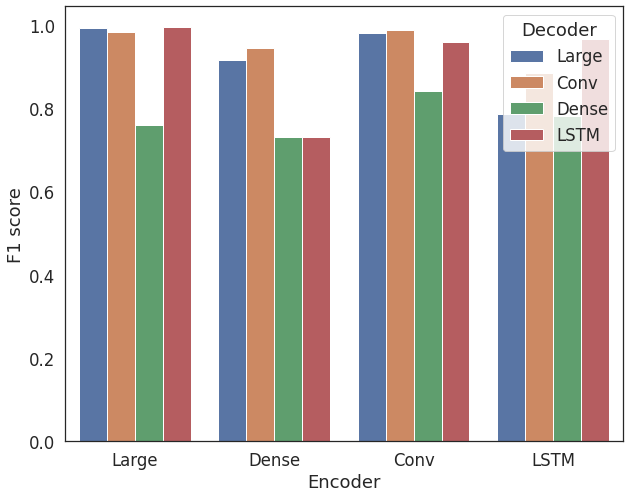}
	\centering	
	\caption{Different architectures for the encoders and the decoder lead to different performance.}
	\label{fig:architectures}
\end{figure}

Architecture with Large encoder and decoder has the best performance. However, simple architectures also performed comparatively, showing that the autodecompose strategy leads the encoders to learn the Desired Semantic Property (DSP). From hereon, we use Architecture Large encoder and decoder for all of our tests. 

\subsection{Testing the sound source encoder}
We test the ability of the model on encoding the sound source and compare the sound source encoder with supervised voice recognition models.

We first pre-trained an autodecompose model on an "unlabeled dataset" (no people id was provided). Then, we used one of the autodecompose encoders to embed samples in a "labeled dataset". We then used a portion (10, 15, 30, 60, or 120 seconds) of the embedded labeled dataset to train a classifier (logistic regression). Finally, we used the held-out data for testing the classifier on the embedding. To demonstrate the opportunity for fine-tuning, we also run experiments in which we "fine-tune" the pre-trained model on the labeled dataset. The fine-tuning takes places \textit{without} access to any labels and performed to enable the model to capture voice and audio characteristics in the the environment of the new dataset.

\subsubsection{Performance comparison}
We used different sources for labeled and unlabeled datasets for our tests, summarized in Table \ref{tbl:tests}.

\begin{table}
	\centering
	\caption{Datasets used}
	\begin{tabular}{|l|l|l|l|l|}
		\textbf{Test} & \textbf{Unlabeled} & \textbf{Labeled} & \textbf{Fine-tuning} & \textbf{Encoder} \\
		1 & LibriSpeech & In-D & No & $E_s$ \\
		2 & LibriSpeech & OoD & No & $E_s$ \\
		3 & LibriSpeech & OoD & Yes & $E_s$ \\
		4 & YouTube & OoD & No & $E_s$ \\
		5 & YouTube & OoD & Yes & $E_s$ \\
		6 & YouTube & OoD & Yes (10s) & $E_s$ \\
		7 & LibriSpeech & OoD & Yes (10s) & $E_s$ \\
		8 & YouTube & OoD & No & $E_c$ \\
		9 & LibriSpeech & OoD & Yes & $E_c$ \\
	\end{tabular}
	\centering		
	\label{tbl:tests}
\end{table}

In this table, "LibriSpeech" refers to the LibriSpeech dataset, "YouTube" refers to a dataset of videos from YouTube, OoD refers to "Out-of-distribution" (a dataset of speakers that is different from the dataset used as the unlabeled data), In-D refers to in-distribution (labeled data and unlabeled data were the same). The "fine-tuning" column indicates whether the autodecompose model was fine-tuned on labeled data before the classification tests. The "Encoder" column indicates which encoder (either $E_s$ or $E_c$) was used to embed the labeled data. In all cases, the fine-tuning took place for 50 iterations. All labeled data was used for fine-tuning (note, fine-tuning does not use the labels but only the signals) in all tests except for tests 6 and 7 where only 10 seconds of labeled data was used for fine tuning (emulating situations where the unlabeled data is also rare). The base models were trained on unlabeled data for 500 epochs \footnote{Pre-trained models are available for download here: \url{https://github.com/rezabonyadi/autodecompose}}. The LibriSpeech samples were taken from the train clean 100 folder in the dataset. The 30 speakers for pre-training the base autodecompose were chosen randomly (close to 12 hours of speech signals), but kept the same across all tests and training. The same is true for people selected for out-of-distribution tests (close to 11 hours of speech). We randomly selected 30 other speakers from LibriSpeech train-clean-100 dataset for testing purposes in this subsection.

Results in Table \ref{tbl:speaker_recog} show that the pre-trained autodecompose model performs well on speaker identification tasks. In Test 1, where the unlabeled and labeled data are taken from the same dataset, the model achieves high average F1 scores, ranging from 99.2 to 99.8. This indicates that the autodecompose model is able to effectively capture the characteristics of the speakers in the dataset without any need for labels when the labeled and unlabeled data come from the same distribution.

\begin{table}
	\centering
	\caption{Average F1 score for recognizing people voices under different settings using pre-trained autodecompose models.}
	\begin{tabular}{cccccc}
		&10 sec&15 sec&30 sec&60 sec&120 sec\\
		Test 1& 99.2& 99.4& 99.6& 99.8& 99.8\\
		Test 2& 93.1& 94.3& 96.5& 97.3& 97.9\\
		Test 3& 98.9& 99.1& 99.4& 99.5& 99.7\\
		Test 4& 82.2& 86.8& 92.2& 94.3& 95.6\\
		Test 5& 96.7& 97.2& 98.3& 98.7& 99.0\\
		Test 6& 90.4& 91.9& 95.3& 96.0& 97.0\\
		Test 7& 94.7& 95.3& 96.8& 97.6& 98.2\\
		Test 8& 25.1& 30.6& 39.4& 44.9& 52.9\\
		Test 9& 22.7& 26.1& 32.5& 37.7& 45.1\\				
	\end{tabular}
	\centering		
	\label{tbl:speaker_recog}
\end{table}

In Tests 2 and 4, where the unlabeled and labeled data are taken from different datasets, the model's performance is slightly lower. In Test 2, where no fine-tuning is performed, the average F1 scores range from 93.1 to 97.9, depending on the length of the labeled data used for training. In Test 4, where the unlabeled data is taken from YouTube videos, the average F1 scores are lower, ranging from 82.2 to 95.6. This suggests that the autodecompose model may not be as effective at capturing the characteristics of speakers from different datasets.  

However, unlabeled fine-tuning of the autodecompose model appears to improve its performance on speaker identification tasks. In Tests 3 and 5, where the autodecompose model is fine-tuned on the out-of-distribution data without labels, the average F1 scores are closer to those in Test 1, indicating that fine-tuning can help the model adapt to the characteristics of the speakers voices in the labeled dataset. In Tests 6 and 7, where the autodecompose model is fine-tuned on a small amount of labeled data (10 seconds only), the average F1 scores are also improved compared to Tests 2 and 4. This suggests fine-tuning even on a small amount of data without any labels has a significant impact on the accuracy of the model.  

Finally, using the $E_c$ encoder for embedding instead of the $E_s$ encoder leads to a significant decrease in performance, as shown in Tests 8 and 9. The average F1 scores are much lower in these tests, indicating that the $E_s$ encoder is more effective at capturing speaker-specific characteristics than the $E_c$ encoder.

Overall, the results suggest that the pre-trained autodecompose model is effective at speaker identification, particularly when the unlabeled and labeled data are taken from the same dataset or the model is fine-tuned on the out-of-distribution data. Using the $E_s$ encoder for embedding is more effective than using the $E_c$ encoder.

Figure \ref{fig:embed_voices} shows the tsne map of encoded labeled signals for Tests 1 to 6.  

\begin{figure}
	\centering
	\begin{tabular}{cc}
		\includegraphics[width=0.23\textwidth]{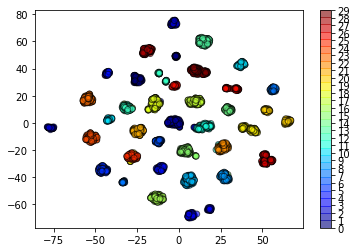}&\includegraphics[width=0.23\textwidth]{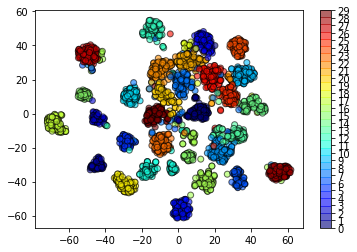}\\
		(a) Test 1&(b) Test 2\\
		\includegraphics[width=0.23\textwidth]{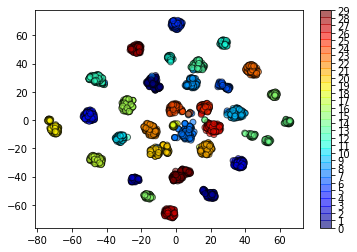}&\includegraphics[width=0.23\textwidth]{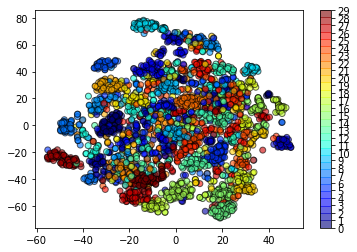}\\
		(c) Test 3&(d) Test 4\\
		\includegraphics[width=0.23\textwidth]{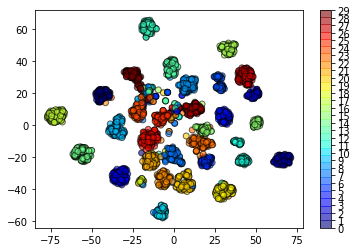}&\includegraphics[width=0.23\textwidth]{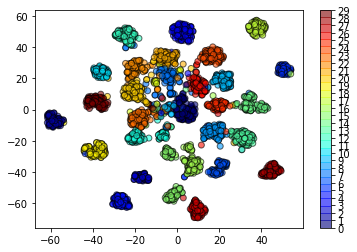}\\
		(e) Test 5&(f) Test 7\\				
	\end{tabular}
	\centering	
	\caption{t-SNE map of the embedding in encoder $ E_s $ under different tests.}
	\label{fig:embed_voices}
\end{figure}

\subsubsection{Comparison with supervised models}
We further compared our model with two models trained in a supervised manner \cite{wan2018generalized,desplanques2020ecapa} and reported the results for recognizing people's voices in audio data in Table \ref{tbl:libri_test}.
 
\begin{table}
	\centering
	\caption{Average F1 score for recognizing people voices under different settings using pre-trained autodecompose models (YouTube).}
	\begin{tabular}{ccccccc}
		model&fine-tuned&10 sec&15 sec&30 sec&60 sec&120 sec\\
		YouTube& No&         83.6& 87.5& 92.3& 94.5& 95.7\\
		YouTube& Yes (10 s)& 90.4& 92.9& 95.3& 96.7& 97.3\\
		YouTube& Yes&        96.3& 96.8& 98.1& 98.7& 99.0\\
		Libri& No&           90.1& 92.9& 94.7& 96.2& 97.2\\					
		Libri& Yes (10 s)&   91.3& 93.3& 95.1& 96.5& 97.3\\			
		Libri& Yes&          \textbf{97.6}& \textbf{98.4}& \textbf{98.9}& \textbf{99.3}& \textbf{99.3}\\			
		GE2E& -&           93.9& 94.8& 95.4& 96.3& 96.9\\	
		TDNN& -&              95.7& 96.8& 97.3& 97.6& 98.0\\
	\end{tabular}
	\centering		
	\label{tbl:libri_test}
\end{table}

The algorithm was tested under different settings, including whether it was fine-tuned on the labeled data (in a self-supervised manner where no labels were used) and the length of the audio used for training the classifier. The unlabled data was either taken from YouTube or from LibriSpeech dataset, train-clean-100 (30 speakers were chosen from this dataset). We tested the methods on LibriSpeech data, test-other dataset (our "labeled" dataset), and presented the results in terms of average F1 score (this dataset contains 40 speakers).

Autodecompose algorithm performs well in this task, with an average F1 score of 83.6-99.3 depending on the specific settings used. This suggests that the self-supervised approach used by the algorithm is effective in recognizing people's voices in audio data. The table also shows that fine-tuning the Autodecompose algorithm on the labeled data leads to significant improvements in performance. For example, when the pre-trained model on LibriSpeech was fine-tuned on the entire labeled dataset, it achieved an average F1 score of 97.6-99.3, which is an improvement of 2.1-7.8 percentage points over the score of 90.1-97.2 achieved when the algorithm was not fine-tuned. Even when the algorithm was fine-tuned on just 10 seconds of labeled data, it still achieved an average F1 score of 91.3-97.3, which is an improvement of 1.3-2.1 percentage points. The gain was higher when the algorithm was pre-trained on YouTube data, where using 10 sec of data for fine-tuning lead to 1.7-8.1 percentage of improvement. These results suggest that fine-tuning the algorithm on a small amount of labeled data can be effective for improving its performance.

In addition, the performance of the Autodecompose algorithm is comparable to that of the supervised models GE2E and TDNN, although our method is self-supervised. In most cases, the F1 scores achieved by the Autodecompose algorithm are similar to or slightly higher than those of the supervised models. This suggests that the self-supervised approach used by the Autodecompose algorithm is a viable alternative to traditional supervised methods for recognizing people's voices in audio data.

\subsubsection{Overfitting test}
To evaluate the resistance of our autodecomposed architecture to overfitting, we tested its generalization ability using a large model pre-trained on a small dataset. The model architecture consisted of 3 conv. (512 filters) + BNR + 3 LSTM (1024 units) + embd1024 as encoder and 2 conv. (1024 filters) + BNR + 2 LSTM (1024 units) + out as decoder. Given the large size of the model and the limited number of data points, this model had the potential to overfit.

To perform the experiment, we selected speech data from three speakers in the LibriSpeech train-clean-100 dataset, each with 20 seconds of utterances, for a total of 60 seconds of speech data. The model was trained from scratch for 5000 epochs on this small data.

\begin{figure}
	\centering
	\includegraphics[width=0.43\textwidth]{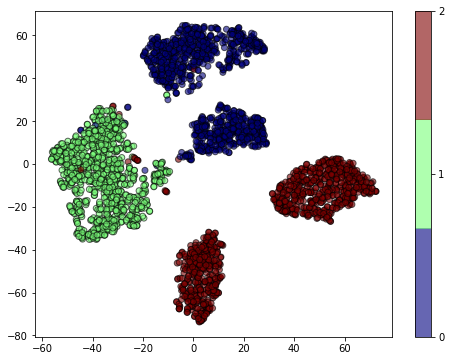}
	\centering	
	\caption{t-SNE map of embedding by $ E_s $ from a very large architecture with large number of embedding dimensions. The model was trained from scratch on a small data set (20 seconds of utterances from three speakers only). This shows the architecture is resistant to over-fitting, even when the model is flexible and number of data points is small.}
	\label{fig:overfit}
\end{figure}

Figure \ref{fig:overfit} visualizes the embedding space of $E_s$ after pre-training from scratch on 60 seconds of speech data. Despite the small number of samples, the voices of the three speakers are clearly distinguishable in the embedding space.

To further evaluate this large pre-training model, we embedded all 4174 seconds of speech from the three speakers using $E_s$. Then, we randomly selected 30 seconds of embedded speech (10 seconds from each speaker) and trained a logistic regression model to predict the speaker ID. The model was evaluated on held-out data and achieved an average F1 score of 98.6\% over 100 runs of the test. This result demonstrates the robustness of our proposed model, as it shows a high tolerance against overfitting even when using a large model with limited data.

\subsection{Content encoder}
We evaluated the performance of autodecompose in embedding audio signals by using the pre-trained encoders ($E_c$ and $E_s$) on a dataset of 30 randomly selected speakers from the LibriSpeech train-clean-100 database. First, we determined if the encoded signals' time-bins were correlated with phonemes in the audio signal by comparing the encoding using $E_c$, $E_s$, and the original mel spectrum. We trained a logistic regression model on 50\% of the embedded utterances, with the phonemes serving as labels, and evaluated the performance using the F1 score on the held-out utterances. We performed this test 100 times on randomly selected utterances for training and testing. The results, shown in Figure \ref{fig:phenoms_comp}, indicate that $E_c$ provides more relevant features for recognizing phonemes compared to $E_s$ or the original mel spectrum.

\begin{figure}
	\centering
	\includegraphics[width=0.45\textwidth]{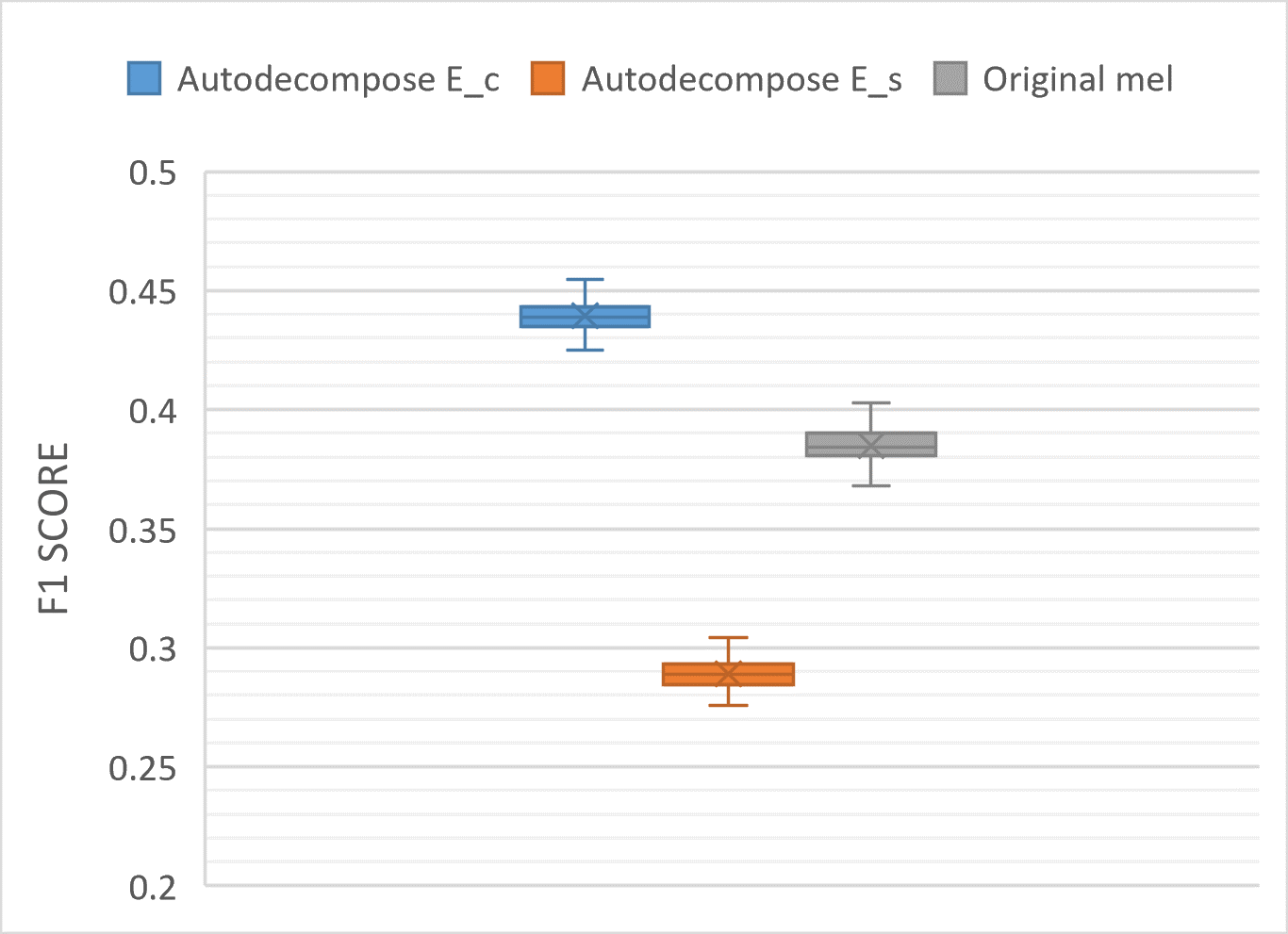}
	\centering	
	\caption{Comparison across using $ E_c $, $ E_s $, and original mel spectrum for recognizing phonemes using a logistics regression. F1 score is the average on the tests.}
	\label{fig:phenoms_comp}
\end{figure}

We also tested the correlation between the encoded audio signals and the words uttered by different speakers. To test this, the audio signals were embedded using $E_s$, $E_c$, or the original mel spectrum. The audio signals were padded to enforce the same length per word, and only the top 80 most frequently used words (out of a total of over 5000 words) were selected. However, the top 40 most frequently used words were excluded because they were short and common words, such as "the", "and", and "can", which could have biased the results. The remaining 40 words were used for comparison, which had a frequency of between 300 to 100 times\footnote{The actual words used were: "more", "has", "did", "now", "than", "only", "our", "some", "about", "little", "well", "two", "after", "upon", "any", "see", "came", "before", "other", "down", "very", "day", "over", "can", "like", "again", "must", "way", "back", "good", "house", "these", "such", "come", "made", "how", "its", "first", "never", "may".}. The results are shown in Figure \ref{fig:embed_speech}, which depicts the t-SNE visualization of the embeddings. It is clear that the features extracted by $ E_s $ are not related to the words uttered or the phonemes (Figure \ref{fig:embed_speech} (e, f)). 

\begin{figure}
	\centering
	\begin{tabular}{cc}
	\includegraphics[width=0.25\textwidth]{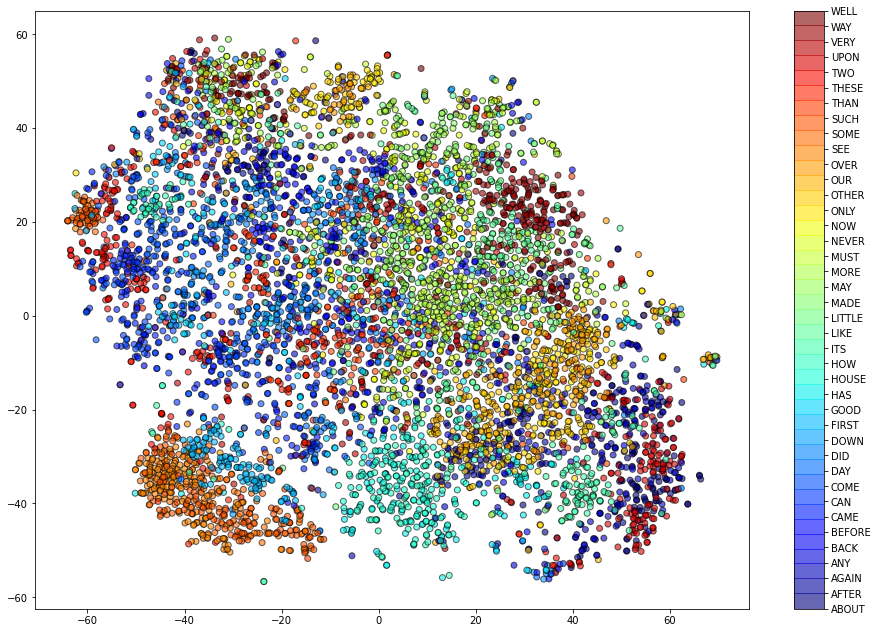}&\includegraphics[width=0.22\textwidth]{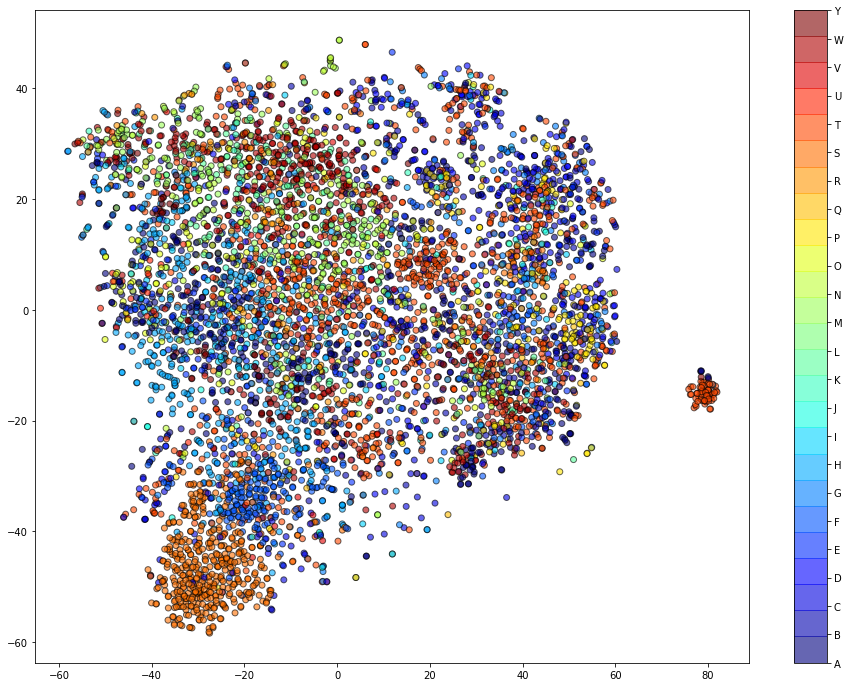}\\
	(a) &(b)\\	
	\includegraphics[width=0.25\textwidth]{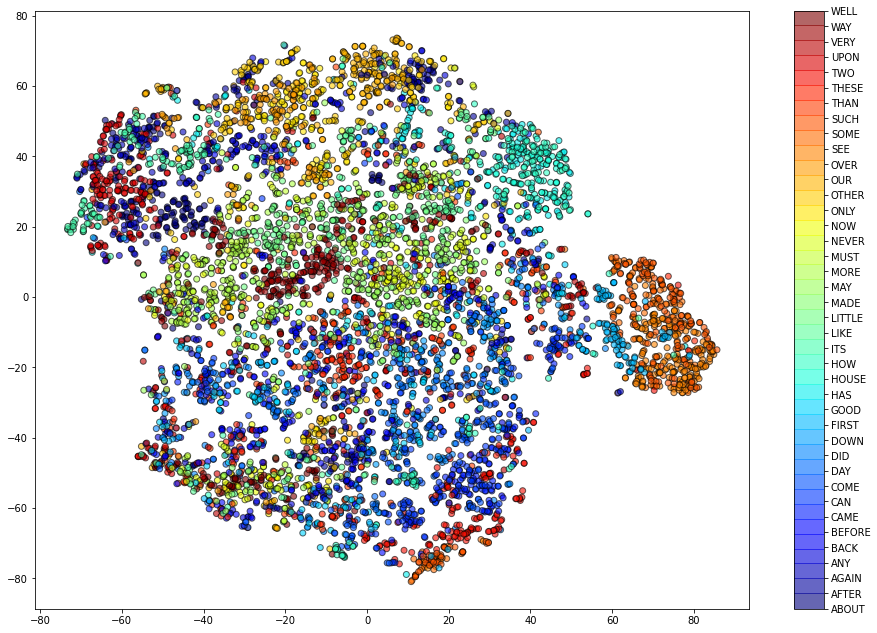}&\includegraphics[width=0.22\textwidth]{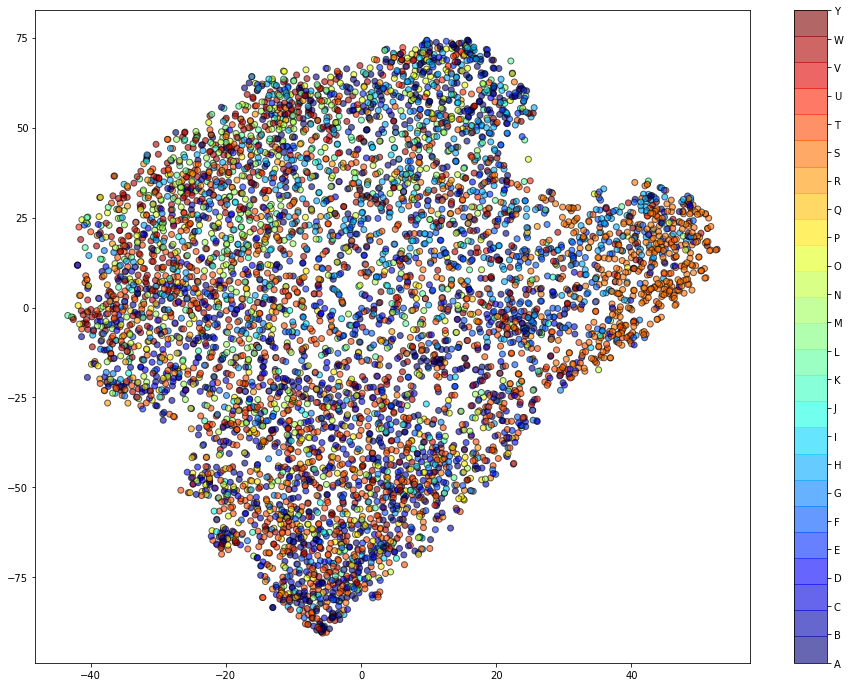}\\
	(c) &(d) \\
	\includegraphics[width=0.25\textwidth]{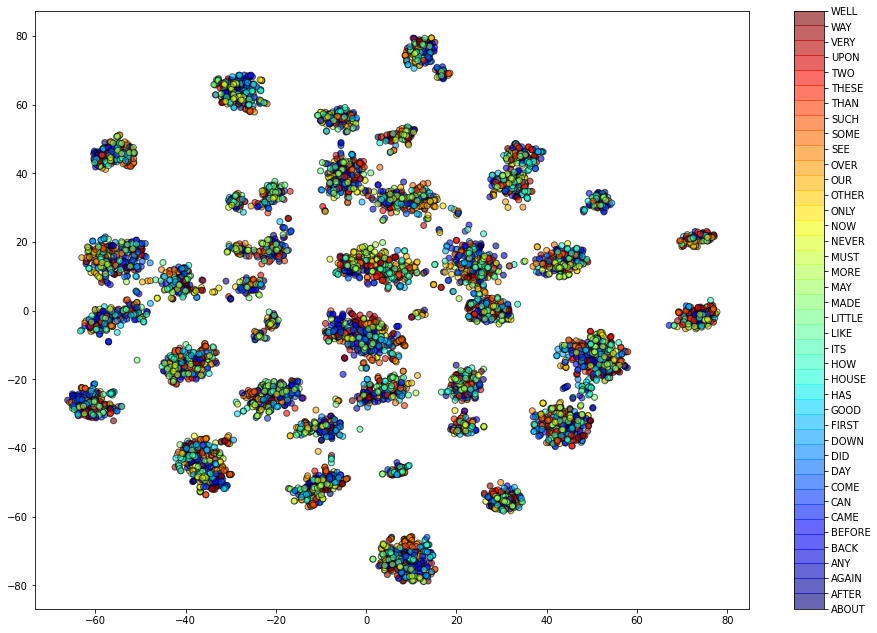}&\includegraphics[width=0.22\textwidth]{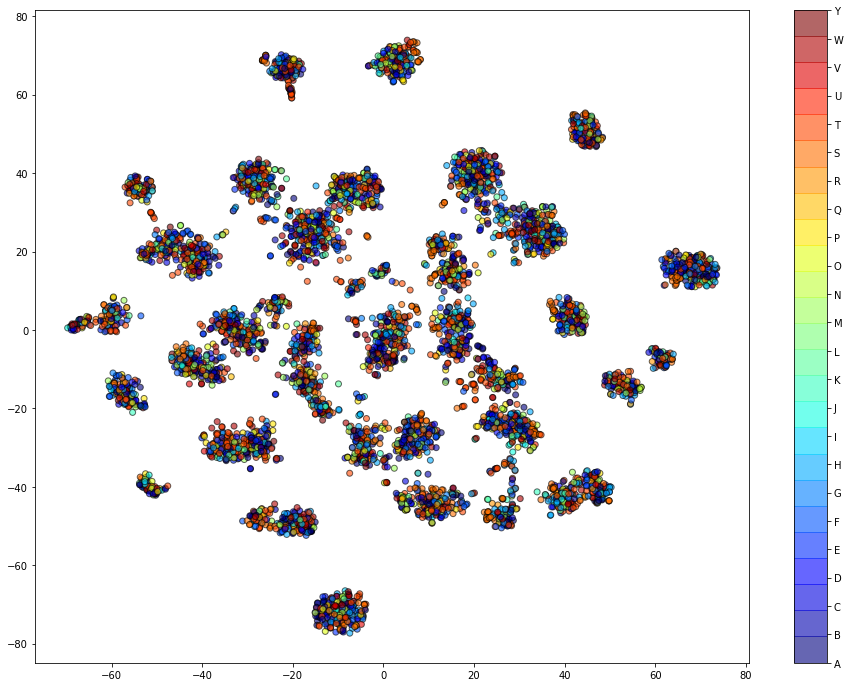}\\
	(e) &(f) \\
	\end{tabular}
	\centering	
	\caption{t-SNE map of (a) selected words embedded by $ E_c $, (b)  phenoms embedded by $ E_c $, (c) selected words represented by original mel , (d) phenoms represented by original mel, (e) selected words embedded by $ E_s $, (f) phenoms embedded by $ E_s $. $ E_c $ has embedded the content successfully while $ E_s $ is correlated with sound source (voice).}
	\label{fig:embed_speech}
\end{figure}

Note that we do not expect the content encoder to have a perfect correlation with words or phonemes uttered. The reason is that the content includes the pitch, the stretch, and the emotion of the utterance in addition to the word uttered. One can use autodecompose recursively to isolate the uttered word using properly designed augmentations. 

\section{Discussion and Future Works}
\label{sec:conclusion}
This paper introduced autodecompose, an architecture that decomposes the raw data to semantically meaningful properties given two complementary augmentations. We proved that the encoders in the autodecompose architecture embed semantically independent properties of real signals, given the augmentations. We tested autodecompose on decomposition of audio signals to sound source (i.e., speaker id) and content of the audio signal. Our experiments showed that the encoder which embeds the speaker voice provides features that are suitable for speaker recognition, comparable (and usually more informative) with supervised learning methods. We also show that the model is less sensitive to overfitting and that the content encoder can be used for speech recognition tasks after further refinements.

\ignore{

\subsubsection*{Acknowledgments}
}
\medskip

\small

\bibliographystyle{IEEEtran}
\bibliography{reference}
\end{document}